\documentclass[runningheads]{llncs}
\usepackage{graphicx}
\usepackage{amsmath}
\usepackage{booktabs}
\usepackage{algorithm}
\usepackage{algorithmic}
\usepackage[switch]{lineno}
\usepackage{amssymb}
\usepackage{enumerate}

\begin{document}
\title{Generally-Occurring Model Change for Robust Counterfactual Explanations}

\author{Ao Xu\orcidID{0009-0009-9071-9070} \and
Tieru Wu\thanks{Corresponding author}\orcidID{0000-0003-3397-1885} }
\authorrunning{A. Xu and T. Wu}

\institute{School of Artificial Intelligence, Jilin University, Changchun, China\\
\email{xuao22@mails.jlu.edu.cn, wutr@jlu.edu.cn}
}

\maketitle           
\begin{abstract}
With the increasing impact of algorithmic decision-making on human lives, the interpretability of models has become a critical issue in machine learning. Counterfactual explanation is an important method in the field of interpretable machine learning, which can not only help users understand why machine learning models make specific decisions, but also help users understand how to change these decisions. Naturally, it is an important task to study the robustness of counterfactual explanation generation algorithms to model changes.
Previous literature has proposed the concept of Naturally-Occurring Model Change, which has given us a deeper understanding of robustness to model change. In this paper, we first further generalize the concept of Naturally-Occurring Model Change, proposing a more general concept of model parameter changes, Generally-Occurring Model Change, which has a wider range of applicability. We also prove the corresponding probabilistic guarantees. In addition, we consider a more specific problem, data set perturbation, and give relevant theoretical results by combining optimization theory.

\keywords{Explainable Artificial Intelligence \and Counterfactual Explanation}
\end{abstract}

\section{Introduction}
\subsection{Background}
E\textbf{X}plainable \textbf{A}rtificial \textbf{I}ntelligence (\textbf{XAI}) \cite{arrieta2020explainable} is a crucial research domain aimed at enhancing the transparency of machine learning algorithms to make their decision-making and predictions more comprehensible and interpretable\cite{machlev2022explainable}. The rise of complex models like deep learning necessitates \textbf{XAI}. These models, often viewed as ``black boxes", present challenges in understanding their internal decision-making processes. Unlike traditional machine learning focused solely on statistical prediction, \textbf{XAI} research centers on explaining the behaviour and outcomes of known AI models.
\textbf{XAI} development encompasses diverse techniques: building interpretable models, calculating feature importance metrics, and creating visualization tools to elucidate model rationale. 
\textbf{XAI} not only aims to enhance the credibility of models but also facilitates compliance with regulatory requirements, enhances user trust, and promotes the transfer of domain knowledge \cite{rasheed2022explainable}. 
Research in this field continues to evolve to ensure that the decisions of machine learning systems become more transparent and trustworthy, addressing societal, legal, and ethical challenges.

Counterfactual explanation method is a significant area within the field of \textbf{XAI}, employed to elucidate the decision-making processes of machine learning and deep learning models \cite{stepin2021survey,wachter2017counterfactual}. This approach aims to address the sensitivity of decision outcomes to input data perturbations.
Counterfactual explanations reveal model behavior by contrasting its predictions on actual and hypothetical inputs.
Ideally, counterfactual instances should have similar features to the input instance being explained, but with a different label. Modifying feature values can shift the counterfactual to a distinct region of the feature space, potentially breaching the decision boundary. This manipulation exposes the model's decision-making process and the impact of decision boundaries on outputs. By comparing actual and counterfactual scenarios, users gain insights into model sensitivity.
Counterfactual explanations enhance the transparency and interpretability of machine learning models, addressing the issue of opacity and increasing model credibility and acceptability. Moreover, counterfactual explanations play a crucial role in domains requiring high levels of interpretability and security, such as healthcare \cite{prosperi2020causal}, finance \cite{gan2021automated}, and autonomous driving \cite{omeiza2021explanations}. Thus, counterfactual explanations hold significant importance in the construction of trustworthy and controllable AI systems.
For various interpretable machine learning models, numerous counterfactual explanation generation algorithms have been developed \cite{brughmans2023nice,guidotti2022counterfactual}.

Current counterfactual explanation methods, however, exhibit limitations in robustness \cite{xu2024weak}. This fragility can lead to inaccurate decision comprehension, misleading explanations, and inappropriate decision modifications. Moreover, explanations lacking robustness may fail to adapt to external circumstances or changes in specific contexts, limiting their breadth and effectiveness in practical applications \cite{artelt2021evaluating,laugel2019issues}. Therefore, researching the robustness of counterfactual explanation generation algorithms has become an urgent task, providing a solid theoretical foundation for constructing more reliable decision support systems. Currently, there have been numerous studies on the robustness of counterfactual generation algorithms. 

\subsection{Related Works}
Saumitra Mishra et al. \cite{mishra2021survey} meticulously categorize the robustness of counterfactual explanation generation algorithms into the following three classes: robustness to model change, robustness to input perturbations and robustness to counterfactual input changes.
Robustness to model change refers to the issue of how the effectiveness of counterfactual explanations evolves when a machine learning model is retrained or when its training parameters are slightly modified. Our paper will focus on the research of this type of robustness. Kaivalya Rawal et al. \cite{rawal2020algorithmic} introduced a novel approach to assess the effectiveness of algorithmic resources in the evaluation of robustness, taking into account the impact of data and model shifts. Through both experimental and theoretical analyses, they elucidated the trade-off between recourse invalidation probability and cost minimized by contemporary recourse generation algorithms.
Emily Black et al. \cite{black2021consistent} concentrated on neural network models and proposed a new method for evaluating the consistency of counterfactual instances in deep models. They also introduced a novel training approach to enhance model consistency.
Andrea Ferrario et al. \cite{ferrario2022robustness} presented a technique termed ``counterfactual data augmentation" to bolster the robustness of counterfactual explanations. To improve the reliability of counterfactual explanations, we augment the training data with counterfactual samples, which enables the model to adapt more effectively to new scenarios.
Simultaneously, Sohini Upadhyay et al. \cite{upadhyay2021towards} 
developed a framework called Robust Algorithmic 
Recourse, which leverages adversarial training to discover resources that exhibit robustness to model variations. The aforementioned endeavors have collectively deepened our understanding of robustness to model change. In particular, our work is closely related to the work of Faisal Hamman et al. \cite{pmlr-v202-hamman23a}. They introduced the concept of Naturally-Occurring Model Change and offered theoretical guarantees, thereby extending the mathematical implications of robustness to model change. Additionally, they defined a stability metric based on this concept and devised two algorithms, T-Rex:I and T-Rex:NN, to generate relatively stable counterfactual explanations. Their experimental results demonstrated superior performance compared to the baseline. Notably, the robustness of counterfactuals is directly related to whether they are defined in a causal sense. Genuine counterfactuals, derived from a causal understanding of data-generating mechanisms, remain robust despite distributional shifts \cite{beckers2022causal}.

\subsection{Our Paper's Contributions} 
In this theoretical paper, our contributions are as follows:
\begin{enumerate}
    \item We first further generalize the concept of Naturally-Occurring Model Change and propose a more universally applicable concept of model parameter change, Generally-Occurring Model Change. We also provide the corresponding probabilistic guarantees.
    \item Additionally, we investigate the specific case of the dataset perturbation problem. By combining the Generally-Occurring Model Change concept with optimization theory, we present novel theoretical results.
\end{enumerate}

\section{Preliminaries}
Our basic notations are consistent with those in \cite{pmlr-v202-hamman23a}. We assume that the instance space is $\mathcal{X}\subseteq \mathbb{R}^d$, $\mathcal{D}$ is a distribution over $\mathcal{Z}=\mathcal{X}\times [0,1]$ and $m(x):\mathcal{X}\longrightarrow  [0,1]$ is the original machine learning model, which can be regarded as a probabilistic model that finally makes decisions based on the results of $\mathbb{I}\left(m(x)>0.5\right )$.  Denote $\mu$ the marginal of $\mathcal{D}$ over $\mathcal{X}$.  In this paper, we always assume that the following Lipschitz condition holds:
\begin{align}
\label{eq:mLip}
    | m (x) - m (y) | \leqslant \gamma_m \cdot \| x - y \|_2 ,
\end{align}
where $x,y\in \mathcal{X}$, $\gamma_m$ is a positive constant and $\|\cdot\|_2$ denotes the Euclidean norm.
\subsection{Counterfactual Explanations}
The solution of counterfactuals is generally associated with some norm or metric on $\mathcal{X}$. In general, we can define the following concept using the norm $ \|  \cdot  \| $ on $\mathcal{X}$:
\begin{definition}[Counterfactuals $\bar x^*$ induced by norm $\|\cdot\|$]\label{def:CE}
     For any $x\in \mathcal{X}$ satisfying $m(x) < 0.5$, its counterfactual with respect to $m(\cdot)$ and $\|\cdot\|$ is:
     \begin{align*}
    \bar x^*:= \underset{\bar x \in \mathbb{R}^d}{\operatorname{argmin }} \    \|x-\bar x\| \text{ s.t. }m(\bar{x})\geq 0.5.
    \end{align*}
\end{definition}
The selection of an appropriate metric or norm is a meaningful problem, and there has been extensive research in the literature on this topic. For example, researchers have studied the selection of the $l_1$-norm \cite{pawelczyk2020counterfactual}, the $l_2$-norm \cite{kenny2021generating}, the Mahalanobis distance \cite{kanamori2020dace}, and so on.

However, as pointed out by \cite{albini2022counterfactual,kanamori2020dace}, counterfactuals solved solely based on a certain metric or norm may deviate significantly from the potential distribution of the data, which is very unreasonable. For example, in a medical scenario, the counterfactual solved requires the height of the input instance to change by $3$ meters, which is completely unrealistic. Therefore, Faisal Hamman et al. \cite{pmlr-v202-hamman23a} defined the following concept:

\begin{definition}[Closest Data-Manifold Counterfactuals $\bar x^*$ induced by norm $\|\cdot\|$ \cite{pmlr-v202-hamman23a}]\label{def:MCE}
     For any $x\in \mathcal{X}$ satisfying $m(x) < 0.5$, its closest data manifold counterfactual with respect to $m(\cdot)$ and $\|\cdot\|$ is:
     \begin{align*}
    \bar x^*:= \underset{\bar x \in \mathcal{X}}{\operatorname{argmin }} \    \|x-\bar x\| \text{ s.t. }m(\bar{x})\geq 0.5.
    \end{align*}
\end{definition}
A closely related metric to the Definition \ref{def:MCE} is the Local Outlier Factor \cite{inproceedings}, which has been widely applied.
\subsection{Gradient Descent}
\textbf{G}radient \textbf{D}escent (\textbf{GD}) is a classical optimization algorithm for finding a local minimum of a differentiable function \cite{ruder2016overview}. In machine learning, \textbf{GD} is used to find the values of a function's parameters that minimize a cost function as far as possible. In this section, we provide some definitions that will be used in Section \ref{sec:GS}.
Given $n$ labeled examples $S=(z_1,z_2,\ldots,z_n)$, where $z_i\in\mathcal{Z}$, consider a decomposable objective function
$
f(\theta)=\frac{1}{n}\sum_{i=1}^n f(\theta;z_i),
$
where $ f(\theta;z_i)$ denotes the loss of model $\theta\in\Theta$ on the example $z_i$. The gradient update $G_{f,\alpha}$ for this problem with  learning rate $\alpha>0$ is given by 
$$
\theta_{t+1}=G_{f,\alpha}(\theta_t):=\theta_{t}-\alpha \nabla_{\theta} f(\theta_t;z_i), 
$$
where the indices $i$ are sequentially selected from $S$.
The following two concepts are useful for analyzing the dynamics of gradient descent.
\begin{definition}[$\tau $-expansive]
    An update rule $G:\Theta\rightarrow\Theta$ is $\tau $-expansive if 
    \begin{align*}
    \sup_{\theta_1,\theta _2\in \Theta}\frac{\left \|G(\theta_1)-G(\theta_2)  \right \|_2 }
    {\left \|  \theta_1-\theta_2\right \|_2 } \leq \tau. 
    \end{align*}
\end{definition}
\begin{definition}[$\sigma$-bounded]
    An update rule $G:\Theta\rightarrow\Theta$ is $\sigma $-bounded if
    \begin{align*}
    \sup_{\theta \in\Theta}\left \| \theta -G(\theta) \right \| _2\leq \sigma.
    \end{align*}
\end{definition}
\begin{remark}
    Generally, we take the norm $\|\cdot\|_{\Theta}$ for the model parameters $\theta$, where $\|\cdot\|_{\Theta}$ is the corresponding norm on $\Theta$. However, in this paper, we take the norm $\|\cdot\|_2$ without loss of generality.
\end{remark}
\section{Main Results}
\subsection{Generally-Occurring Model Change}
A popular assumption in existing literature \cite{upadhyay2021towards} is
\begin{align}
    \| \text{Params} (M) - \text{Params} (m) \| < \Delta \text{ for a constant }  \Delta . \label{eq:para}
\end{align}
As Faisal Hamman et al. \cite{pmlr-v202-hamman23a} point out, the assumption that changes in model parameters always lead to changes in classification performance is inappropriate. Instead, we should focus more on changes in the classification performance of the model rather than the changes in the parameters themselves.
Their revision of Equation (\ref{eq:para}) is similar to the concept of equivalence classes in pure mathematics: those changes in model parameters that do not change the shape of the decision boundary should be tolerated. Specifically, Faisal Hamman et al. define Naturally-Occurring Model Change as follows:
\begin{figure*}
    \centering
    \includegraphics[width=1\linewidth]{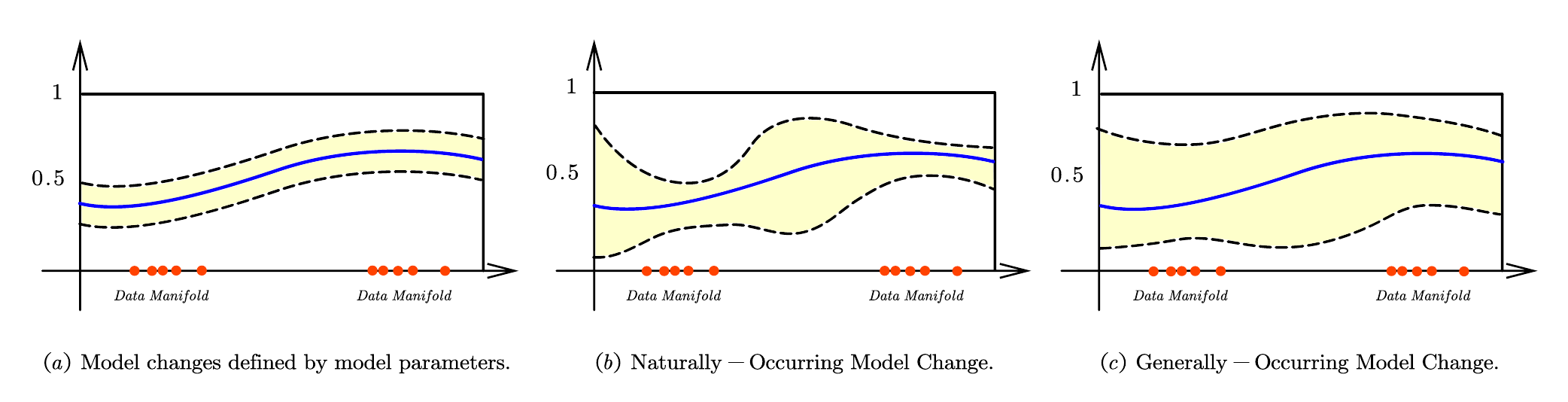}
    \caption{Comparison of Three Model Changes. (c) corresponds to the model change defined by Definition \ref{def:GOMC}. It builds on (b) by allowing the distribution of M to be not centered at $m(\cdot)$, thus having a larger range of model changes.}
    \label{fig:GOMC}
\end{figure*}
\begin{definition}[Naturally-Occurring Model Change, \cite{pmlr-v202-hamman23a}]\label{def:NOMC}
    A naturally-occurring model change is defined as follows:
    \begin{enumerate}
        \item Expectation condition. We assume that the following equation holds:
        $$\mathbb{E}[M(X)|X=x]=\mathbb{E}[M(x)]=m(x),$$ where the expectation is over the randomness of $M(\cdot)$ given a fixed value of $X=x\in \mathbb R^d$.
        \item Variance condition. We assume that the following inequality holds: $$\textit{Var}[M(x)|X=x]=Var[M(x)]=\nu_x  \leq \nu$$ which define on the fixed value of $X=x\in\mathbb{R}^d$. Futhermore, whenever 
 $x$ lies on the data manifold $\mathcal{X}$, we have $\nu_x\leq \nu$, where $\nu$ is a small constant.
        \item Lipschitz condition. Whenever $m(x)$ is $\gamma_m$-Lipschitz, any updated model $M(\cdot)$ is also $\gamma$-Lipschitz. Note that, this constant $\gamma$ does not depend on $M(\cdot)$ since we may define $\gamma$ to be an upper bound on the Lipschitz constants for all possible $M(\cdot)$ as well as $m(\cdot)$.
    \end{enumerate}
\end{definition}
It is worth noting that the first constraint in Definition \ref{def:NOMC} is a bit strict and inconvenient to verify. We generalize the concept of naturally-occurring model change and propose the following definition:
\begin{definition}[Generally-Occurring Model Change, Main Concept]\label{def:GOMC} 
A generally-occurring model change is defined as follows:
    \begin{enumerate}
        \item Expectation condition. We assume that the following inequality holds: $$\mathbb{E}_M \big[ \left \| m-M \right \|_{L^2(\mu)}   \big]\leq \delta,$$ where the expectation is over the randomness of $M(\cdot)$.
        \item Subgaussian condition. We assume that  $\left \| m-M \right \|_{L^2(\mu)}$ is a $\nu^2$-subgaussian random variable, which is equivalent to $\phi(\lambda)\leq \lambda^2\nu^2/2$, where $\phi(\lambda)$ is the log-moment generating function of $\left \| m-M \right \|_{L^2(\mu)}$.
        \item Lipschitz condition. Whenever $m(x)$ is $\gamma_m$-Lipschitz, any updated model $M(\cdot)$ is also $\gamma$-Lipschitz.This constant $\gamma$ does not depend on $M(\cdot)$ as well.
    \end{enumerate}
\end{definition}

\textbf{Comparison of Three Model Changes.} The intuitive comparison of the three model changes mentioned in this section is shown in Fig. \ref{fig:GOMC}.
The blue curve in the Fig. \ref{fig:GOMC} represents the original model $m(\cdot)$ and the yellow region represents the range in which the model can change.
Fig. \ref{fig:GOMC} (a) corresponds to the model change defined by Equation (\ref{eq:para}). This type of model change requires that the model parameters are bounded, so the model can only change around $m(\cdot)$.
Fig. \ref{fig:GOMC} (b) corresponds to the model change defined by Definition \ref{def:NOMC}. Under this model change, $M(\cdot)$ is centered at $m(\cdot)$ and changes in a low-variance manner around $m(\cdot)$. It allows for large-scale changes in parameters that do not affect the predictions on the data manifold.
Fig. \ref{fig:GOMC} (c) corresponds to the model change defined by Definition \ref{def:GOMC}. It builds on Fig. \ref{fig:GOMC} (b) by allowing the distribution of $M(\cdot)$ to be not centered at $m(\cdot)$, thus having a larger range of model changes. Therefore, Definition \ref{def:GOMC} can cover a wider range of model variations, making it a more general concept.
\subsection{Probabilistic Guarantee}
For naturally-occurring model change, Faisal Hamman et al. \cite{pmlr-v202-hamman23a} proposed the following probabilistic guarantee:
\begin{theorem}[Probability Guarantees under Naturally-Occurring Model Change, \cite{pmlr-v202-hamman23a}]\label{thm1}
    Let $X_1,X_2,\ldots,X_k$ be $k$ i.i.d. random variables with distribution $\mathcal{N}(x,\sigma^2 I_d)$ and $\psi_M^k=\frac{1}{k} \sum_{i=1}^k \left ( m(X_i)-M(X_i) \right ) $. Suppose $\left |  \mathbb{E}[\psi_M^k|M]-\mathbb{E}[\psi_M^k] \right |<\varepsilon' $. Then, for any $\varepsilon>2\varepsilon'$, a counterfactual $x\in\mathcal{X}$ under naturally-occurring model change satisfies:
    \begin{align}\label{eq:PG1}
    \operatorname{Pr} \left(M(x)\leq R_{k,\sigma ^2}(x,m)-\varepsilon \right) \leq
     \exp \left( \frac{- k \varepsilon^2}{8 (\gamma + \gamma_m)^2 \sigma^2}
    \right),
    \end{align}
    where $R_{k,\sigma^2}(x,m)=\frac{1}{k}\sum_{x_i\in N_{x,k} } \left ( m(x_i)-\gamma \cdot \left \| x-x_i\right \|_2  \right )$, $N_{x,k}$ is a set of $k$ points drawn from the Guassian distribution $\mathcal{N}(x,\sigma I_d)$ with $I_d$ being the identity matrix. The probability is over the randomness of both $M$ and $X_i$'s.
\end{theorem}
In the Equation (\ref{eq:PG1}), $R_{k,\sigma^2}(x,m)$ is the stability measure proposed by Faisal Hamman et al.. They point out that in the concrete computation, because $\gamma$ is unknown, only the relaxed version of $R_{k,\sigma^2}$ can be computed: 
$$
\hat{R}_{k,\sigma^2}=\frac{1}{k}\sum_{x_i\in N_{x,k} } \left ( m(x_i)-\left |m(x)-m(x_i)  \right |   \right ) .
$$
Meanwhile, they proposed the Counterfactual Robustness Test: $\hat{R}_{k,\sigma^2}\geq \tau.$
According to Theorem \ref{thm1}, setting a larger $\tau$ will  guarantee that the new model $M(\cdot)$ has a higher validity with high probability.
To better understand Theorem \ref{thm1}, Equation (\ref{eq:PG1}) can be written as follows:
\begin{align}
\Pr \left(\frac{1}{k} \sum_{i = 1}^k m (X_i) - M (x) \geq \frac{\gamma}{k} \sum_{i
= 1}^k \| x - X_i \|_2 +\varepsilon \right) 
\leq
 \exp \left( \frac{- k \varepsilon^2}{8 (\gamma + \gamma_m)^2 \sigma^2}\label{eq:PG1E}
\right).
\end{align}
Equation (\ref{eq:PG1E}) provides an intuitive understanding of Theorem \ref{thm1}: the prediction of the new model $M(\cdot)$ for $x$ is approximately equal to the average prediction of the original model $m(\cdot)$ for $x$ in its neighborhood.

For the more general concept of generally-occurring model change, we provide the following probabilistic guarantee:
\begin{theorem}[Probability Guarantees under Gernerally-Occurring Model Change, Main Theorem]\label{thm2}
    Let $X_1,X_2,\ldots,X_k$ be $k$ i.i.d. random variables with distribution $\tilde \mu\ll \mu, \kappa(\tilde \mu,\mu):=\| d \tilde{\mu} / d \mu \|_{L^2 (\mu)}$.  Then, for any $\varepsilon,\ell>0$, a counterfactual $x\in\mathcal{X}$ under gernerally-occurring model change satisfies:
   \begin{align}
     \Pr& \bigg( \frac{1}{k} \sum_{i = 1}^k  m (X_i) - M (x)   \geq \frac{\gamma}{k}\sum_{i
      = 1}^k \| x - X_i \|_2 + \varepsilon \notag\\
    +&(\delta +\ell\cdot \nu) \cdot
      \kappa (\tilde{\mu},\mu)\bigg)
    \leq 2 \exp \left(- \frac{\varepsilon^2 k}{2}  \right)+\exp\left (  -\frac{\ell^2}{2} \right ).\label{eq:PG2}
    \end{align}
    The probability is over the randomness of both $M$ and $X_i$'s. 
\end{theorem}

\textbf{Theorems Comparison.} A straightforward comparison of Equations (\ref{eq:PG1E}) and Equation (\ref{eq:PG2}) shows that Theorem \ref{thm2} is an extension of Theorem \ref{thm1}. Specifically, this extension is twofold:
\begin{enumerate}[i.]
    \item  First, Theorem \ref{thm2} does not require that the sampling distribution used near $x$ is Gaussian. The extension is particularly useful when the true distribution of the data near $x$ is significantly different from the Gaussian distribution $\mathcal{N}(x,\sigma I_d)$. 
    \item Second, in Theorem \ref{thm2}, we no longer need the technical assumption $$\left |  \mathbb{E}[\psi_M^k|M]-\mathbb{E}[\psi_M^k] \right |<\varepsilon' $$ which is not natural in Theorem \ref{thm1}.
\end{enumerate}
\subsection{Proof of Theorem \ref{thm2}}\label{pfmr}
Before proving Theorem \ref{thm2}, we present the following lemma:
\begin{lemma}[Deviation Bound]\label{lem:DB}
    Let $S=$ $\{X_1,X_2,\ldots,X_k\}$ consist of $k$ i.i.d. random variables belonging to $\tilde \mu\ll \mu$ and 
    $ \kappa(\tilde \mu,\mu):=\| d \tilde{\mu} / d \mu \|_{L^2 (\mu)},\psi_M^k := \frac{1}{k} \sum_{i  = 1}^k 
    (m (X_i) - M (X_i)) $. Then under gernerally-occurring model change, we have 
   $
         \operatorname{Pr} \left ( \psi_M^k\geq \| m - M \|_{L^2 (\mu)} \cdot \kappa (
                  \tilde{\mu}, \mu)+ \varepsilon|M=\tilde m \right ) 
        \leq 2 \exp (- \varepsilon^2 k/2 ),
   $ 
    where $\tilde{m}$ is a model that may be obtained after the change.
\end{lemma}
\begin{proof}
    It is noted that for any $i = 1, 2,\ldots, k$, we have
    \begin{align*}
      & | \psi_M^k (X_1, X_2, \ldots, X_i, \ldots, X_K) 
    - \psi_M^k (X_1, X_2, \ldots,
      X_i', \ldots, X_K) | \\
      = & \left | \frac{1}{k} ( m (X_i) - M (X_i) ) - \frac{1}{k}  (m (X_i') - M (X_i') ) \right | 
      \leq \frac{2}{k} .
    \end{align*}
    Hence, by McDiarmid's Inequality \cite{mcdiarmid1989method}, we have
    \begin{align}
    \Pr (\psi_M^k - \mathbb{E}_{S} [\psi_M^k |M = \tilde{m}] \geq \varepsilon |M =
      \tilde{m})
      \leq   2  \exp \left(- \frac{\varepsilon^2 k}{2}  \right).\label{eq:DB1}
    \end{align}
    Furthermore, by using Triangle inequality and Hölder's inequality, we have
     \begin{align*}
    \mathbb E_S\left [ \psi_M^k\mid M=\tilde{m} \right ] 
    =&\mathbb E_S\left [ \frac{1}{k} \sum_{i  = 1}^k  (  m (X_i) - \tilde{m} (X_i)  )  \right ]
    \\ =&\frac 1k \sum_{i  = 1}^k\mathbb E_{X_i}\big [ \left (  m (X_i) - \tilde{m} (X_i)  \right )  \big ]
    (\text{ i.i.d. })\\   
    =&\int_{\mathcal{X} }  (m (X) - \tilde{m} (X) ) d \tilde{\mu} (X)
\\
       \leq&  \int_{\mathcal{X} } | m (X) - \tilde{m} (X) | d \tilde{\mu} (X)\text{ ( Triangle inequality ) }\\
       =&  \int_{\mathcal{X} } | m (X) - \tilde{m} (X) | \cdot \frac{d \tilde{\mu}}{d \mu} \cdot d
      \mu (X)\\
       \leq& \| m - \tilde{m} \|_{L^2 (\mu)} \cdot \kappa (
      \tilde{\mu}, \mu)\text{ ( Hölder's inequality ) } .
    \end{align*}
    Therefore, we have
    \begin{align}
      \left \{ \psi_M^k-\| m - \tilde{m} \|_{L^2 (\mu)} \cdot \kappa (
          \tilde{\mu},\mu)\geq \varepsilon \right \} 
    \subseteq  \left \{ \psi_M^k-\mathbb{E}_{\tilde \mu}[\psi_M^k|M=\tilde m]\geq \varepsilon   \right \}.\label{eq:DB2}
    \end{align}
    By combining Equation (\ref{eq:DB1}) and Equation (\ref{eq:DB2}), we obtain 
    \begin{align*}
     \operatorname{Pr} \left ( \psi_M^k-\| m - M \|_{L^2 (\mu)} \cdot \kappa (\tilde{\mu}, \mu)\geq \varepsilon|M=\tilde m \right ) 
    \leq 2  \exp \left(- \frac{\varepsilon^2 k}{2}  \right).
    \end{align*}
    This completes the proof. $\hfill \square$
\end{proof}
Subsequently, we complete the proof of Theorem \ref{thm2}.
\begin{proof}[proof of Theorem \ref{thm2}]
    The expectation condition implies that:
\begin{align*}
\left \{ \| m - M \|_{L^2 (\mu)} -\delta> \ell\cdot \nu \right \} 
\subseteq \left \{ \| m - M \|_{L^2 (\mu)}-\mathbb{E}_M\left [  \| m - M \|_{L^2 (\mu)}\right ]   >\ell\cdot \nu \right \} 
\end{align*}
Then, using Chernoff bound \cite{hellman1970probability}, we have
\begin{align}
&\Pr (\| m - M \|_{L^2 (\mu)} > \delta +\ell\cdot \nu )\notag
\\
 \leq& \Pr \bigg(\| m - M \|_{L^2 (\mu)}-\mathbb{E}_M\left [  \| m - M \|_{L^2 (\mu)}\right ]   >\ell\cdot \nu \bigg)\notag  \\
  \leq & \exp\left ( -\left (  \ell\cdot \nu   \right ) ^2/2\nu^2 \right )=\exp\left (  -\ell^2/2\right ).   \label{eq:thm2MI}
\end{align}

 And, notice that
\begin{align}
& \Pr (\psi_M \geq \varepsilon + (\delta +\ell\cdot \nu) \cdot
  \kappa (\tilde{\mu}, \mu))\notag\\
  = & \mathbb{E}_M \left [ \Pr (\psi_{\tilde{m}} \geq \varepsilon
  {+ (\delta +\ell\cdot \nu) \cdot \kappa
  ( \tilde{\mu}, \mu) |M = \tilde{m}}) \right ] \notag\\
  = & \mathbb{E}_M \big [\Pr (\psi_{\tilde{m}} \geq \varepsilon
  + (\delta +\ell\cdot \nu) \cdot \kappa
  ( \tilde{\mu}, \mu) |M = \tilde{m})\notag 
\cdot \mathbb I \left (  \| m - M \|_{L^2 (\mu)}
  \leq \delta +\ell\cdot \nu\right )\big]\notag \\
  + & \mathbb{E}_M \big [\Pr (\psi_{\tilde{m}} \geq \varepsilon
  + (\delta +\ell\cdot \nu) \cdot \kappa
  ( \tilde{\mu}, \mu) |M = \tilde{m}) 
\cdot \mathbb I \left (  \| m - M \|_{L^2 (\mu)}
  > \delta +\ell\cdot \nu\right) \big] \notag\\
\leq & \mathbb{E}_M \big [\Pr (\psi_{\tilde{m}} \geq \varepsilon
  +  \| m - M \|_{L^2 (\mu)} \cdot \kappa
  ( \tilde{\mu}, \mu) |M = \tilde{m})\cdot \mathbb I \left (  \| m - M \|_{L^2 (\mu)}
  \leq \delta +\ell\cdot \nu\right )\big]\notag\\
  + & \operatorname{Pr}  \left (  \| m - M \|_{L^2 (\mu)}
  > \delta +\ell\cdot \nu\right ),\notag
\end{align}
which, combined with  Lemma \ref{lem:DB} and Equation (\ref{eq:thm2MI}), gives
\begin{align}
& \Pr (\psi_M \geq \varepsilon + (\delta +\ell\cdot \nu) \cdot
  \kappa (\tilde{\mu}, \mu))\notag\\
 \leq & 2 \exp \left(- \frac{\varepsilon^2 k}{2}  \right)\cdot\Pr (\| m - M \|_{L^2 (\mu)} \leq (\delta +\ell\cdot \nu))
\notag\\
+&\operatorname{Pr}  \left (  \| m - M \|_{L^2 (\mu)}
  > \delta +\ell\cdot \nu\right )\notag\text{( Lemma \ref{lem:DB} )}\\
\leq& 2  \exp \left(- \frac{\varepsilon^2 k}{2}  \right)+\exp\left (  -\frac{\ell^2}{2} \right ).\text{ ( Equation (\ref{eq:thm2MI}) )}\label{eq:thm2C}
\end{align} 
Now, note that by the Lipschitz condition, we have
\begin{align*}
  & \frac{1}{k} \sum_{i = 1}^k   m (X_i) - M (x)   - \frac{\gamma}{k} \sum_{i
  = 1}^k \| x - X_i \|_2 \\
  \leq & \frac{1}{k} \sum_{i = 1}^k   m (X_i) - M (x)  - \frac{1}{k}
  \sum_{i = 1}^k  \left ( M (X_i) - M (x) \right )  \\
  = & \frac{1}{k} \sum_{i = 1}^k \left (  m (X_i) - M (X_i)\right ) =\psi_M^k.
\end{align*}
Therefore, 
\begin{align}
&\left \{ \psi_M \geq \varepsilon + (\delta +\ell\cdot \nu) \cdot
  \kappa (\tilde{\mu}, \mu) \right \} \notag\\
\subseteq  &\bigg \{ \frac{1}{k} \sum_{i = 1}^k   m (X_i) - M (x)   \geq \frac{\gamma}{k}\sum_{i
  = 1}^k \| x - X_i \|_2 
+(\delta +\ell\cdot \nu) \cdot
  \kappa (\tilde{\mu}, \mu)+ \varepsilon
\bigg \}.\label{eq:thm2set}
\end{align}
Finally, combining Equation (\ref{eq:thm2C}) and Equation (\ref{eq:thm2set}) yields
\begin{align*}
 \Pr& \bigg( \frac{1}{k} \sum_{i = 1}^k  m (X_i) - M (x)   \geq \frac{\gamma}{k}\sum_{i
  = 1}^k \| x - X_i \|_2 + \varepsilon \\
+&(\delta +\ell\cdot \nu) \cdot
  \kappa (\tilde{\mu}, \mu)\bigg)
\leq 2 \exp \left(- \frac{\varepsilon^2 k}{2}  \right)+\exp\left (  -\frac{\ell^2}{2} \right ).
\end{align*}
This completes the proof.$\hfill \square$
\end{proof}

\section{Case Study: Dataset Perturbation Problem}\label{sec:GS}
\subsection{Problem Setting and Assumptions}
In this section, we consider a specific problem. Suppose we have an original dataset $S_1$ of size $n$, and a new dataset $S_2$ obtained by perturbation, as shown in Fig. \ref{fig:DSP}. The two datasets have most of their elements in common, with only a few $r$ elements being different. Thus, we can write $S_1,S_2$ as 
\begin{align*}
S_1=\{ z_1, \ldots, z_{n - r} \}\bigcup  \{ z_{n - r + 1} \ldots z_n\},
S_2 = \{ z_1, \ldots, z_{n - r} \}\bigcup \{ s_{n - r + 1} \ldots s_n\}.
\end{align*}
\begin{figure}
    \centering
    \includegraphics[width=0.6\linewidth]{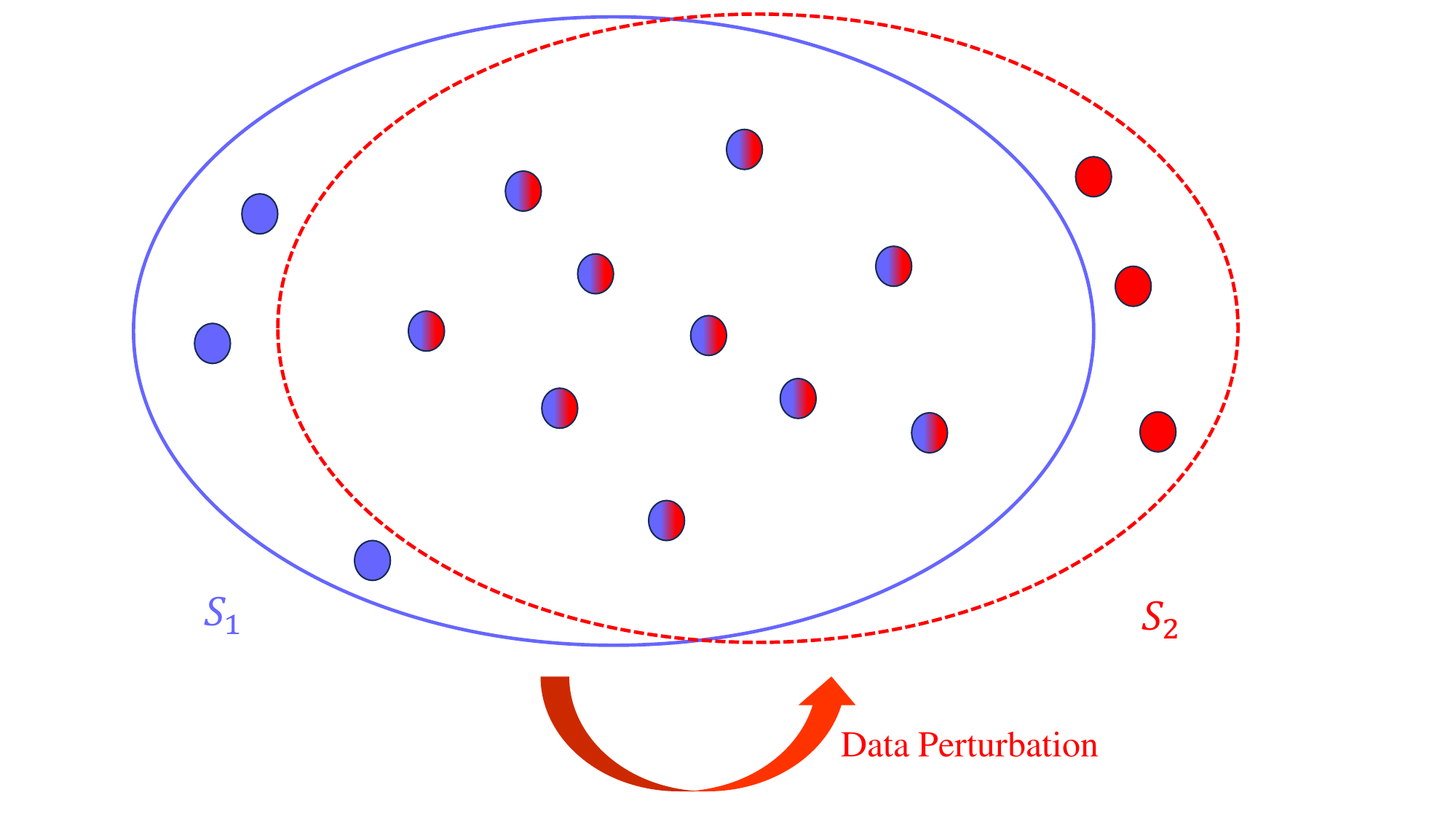}
    \caption{Data Perturbation: A small portion of the original dataset is modified to create a new dataset. The purple circles in the figure represent the data points from the original dataset, the red circles represent the data points from the new dataset, and the bicolored circles represent the data points that are shared between the two datasets.}
    \label{fig:DSP}
\end{figure}
Now, we assume that we start from the same initialization parameters, first train the original model $m(\cdot)$ using $ S_1$ through \textbf{GD}, and then train the new model $M(\cdot)$ using $S_2$ through \textbf{GD}.
We consider the following question: when model $m(\cdot)$ is updated to $M(\cdot)$, how can we theoretically characterize the robustness of counterfactuals under this model change?
In fact, we will see in Section \ref{secTR} that the model change from $m(\cdot)$ to $M(\cdot)$ cannot be considered as a naturally-occurring model change, but it can still be considered as a generally-occurring model change as defined in this paper. Furthermore, we can use Theorem \ref{thm2} to provide probability guarantees.

However, before we proceed with the theoretical argument, we need introduce some reasonable assumptions. More precisely, we assume the following:
\paragraph{ Assumption 1.}
    In the rest of this paper, we always assume that the model under consideration satisfies the Lipschitz condition.
\paragraph{ Assumption 2.}
    In the present paper, we consider a loss function $f$ that satisfies the following conditions for all $z\in \mathcal{Z},\theta_1,\theta_2\in \Theta$:
    \begin{enumerate}
	 \item $f(\cdot,z)$ is $L$-Lipschitz :
        $
         |f(\theta_1,z)-f(\theta_2,z)|\leq L \|\theta_1-\theta_2\|_2
        $.
        \item $f(\cdot,z)$ is $\alpha$-smooth :
        $\left \| \nabla f(\theta_1)-\nabla f(\theta_2)  \right \|_2
        \leq \alpha \left \| \theta_1-\theta_2 \right \|_2  $.
        \item  $f(\cdot,z)$ is left $\xi$-admissible :
         $|f(\theta_1,z)-f(\theta_2,z)|
        \geq \xi \left | c_{\theta_1}(z)-c_{\theta_2}(z) \right |$, 
         where $c_{\theta_1}(\cdot)$ and $c_{\theta_2}(\cdot)$ are the corresponding classification models for $\theta_1$ and $\theta_2$.
    \end{enumerate}
To demonstrate the validity of Assumption 2, we consider an example of a specific loss function.
\begin{proposition}
Let instance space $\mathcal{X}=\{x\in\mathbb R^d:\|x\|_2\leq B\}$, hypothesis space $\Theta=\{\theta\in\mathbb R^d:\|\theta\|_2\leq 1\}$, and 
 label space $\mathcal{Y}=\{-1,+1\}$. Then the logistic loss function 
 $f (\theta, z) = \ln (1 + \exp\left ( - y (x^{\top } \cdot\theta) \right ) )$
 satisfies Assumption 2, for
 $
 L=B/(\exp(-B)+1),\alpha=B^2/4,\xi=1/\exp(B)+1
 $.
\end{proposition}
\begin{proof}
    Let $yx^\top\cdot \theta$ be denoted by $t\in [-B,B]$. Then, 
    $
    f(\theta,z)=l(t)=\ln(1+e^t),
    $
     Note that
     $
      | l' (t) |  =  \frac{1}{{e^t}  + 1} \in \left[ \frac{1}{e^B + 1},
  \frac{1}{e^{- B} + 1} \right] 
     $.
     By the mean value theorem and Cauchy-Schwarz inequality, there is some $\tau\in [-B,B]$ such that
     \begin{align*}
  | l (t_1) - l (t_2) |  & =  | l' (\tau) | \cdot | t_1 - t_2 |
    \leq  \frac{1}{e^{- B} + 1} \| x \|_2 \cdot\|\ \theta_1 - \theta_2 \|_2
   \leq  \frac{B}{e^{- B} + 1} \| \theta_1 - \theta_2 \|_2 .
 \end{align*}
Consequently, $f(\cdot,z) $ is $L$-Lipschitz.
Similarly, for any $\theta_1,\theta_2\in\Theta$, we let $t_1=x^\top\cdot \theta_1,t_2=x^\top\cdot \theta_2$. There is some $\tau\in [-B,B]$ such that
\begin{align*}
  | l (t_1) - l (t_2) | & =  | l' (\tau) | \cdot | t_1 - t_2 |
   \geq  \frac{1}{e^B + 1} \left | x^{\top} \cdot \theta_1 - x^{\top} \cdot
  \theta_2 \right | .
\end{align*}
Hence, $f(\cdot,z) $ is also left $\xi$-admissible. 
Furthermore, since
\begin{align*}
  \left | l'' (t) \right |  & =  \frac{e^{- x}}{(1 + e^{- x})^2}
   =  \frac{1}{(1 - e^{- x}) (1 + e^x)} \leqslant \frac{1}{4} .
\end{align*}
We can obtain the following by utilizing the same techniques:
\begin{align*}
 \| \nabla  f (\theta_1, z) - \nabla  f (\theta_2, z) \|_2
  \leq \frac{B}{4} \| x \|_2 \cdot \| \theta_1 - \theta_2 \|_2 \leq
  \frac{B^2}{4} \| \theta_1 - \theta_2 \|_2 .
\end{align*}
Thus, $f(\cdot,z) $ is also $\alpha$-smooth. $\hfill\square$
\end{proof}
\subsection{Theoretical Result}\label{secTR}
We consider the setting of convex optimization and prove the following theorem:
\begin{theorem}[Convex Optimization]\label{thm3}
    Suppose that the loss function $f(\cdot,z)$ is convex for all $z\in\mathcal{Z}$ and satisfies Assumption 2, and that $m(\cdot)$ and $M(\cdot)$ satisfy Assumption 1. Suppose that we run \textbf{GD} with step sizes $\eta_t\leq 2/\beta $ for $ n$ steps. Then,
    \begin{align*}
     \Pr& \Bigg( \frac{1}{k} \sum_{i = 1}^k  m (X_i) - M (x)   \geq \frac{\gamma}{k}\sum_{i
      = 1}^k \| x - X_i \|_2 + \varepsilon \notag\\
    +& \left ( \frac{2 L^2}{\xi} \sum_{t = n
    - r + 1}^{n} \eta_t \right )  \cdot
      \kappa (\tilde{\mu}, \mu)\Bigg)
    \leq 2  \exp \left(- \frac{\varepsilon^2 k}{2}  \right).
    \end{align*}
\end{theorem}
\begin{proof}
    Consider the gradient updates $ G_1^m,G_2^m,\ldots,G_n^m$ and $G^M_1,G_2^M,\ldots,G^M_n$ induced by running \textbf{GD} on $S_1,S_2$. For \textbf{GD}, $m$ and $M$ are parameterized as $\theta^m = \theta^m_{n+1} = G_n^m(\theta^m_n),\theta^M = \theta^M_{n+1} = G_n^M(\theta^M_n)$.
    We define $\delta_t = \| \theta^m_t - \theta^M_t \|_2$,  where $\theta^m_t,\theta^M_t $ are the model parameters trained on $S_1$ and $S_2$ after the $(t-1)$-th gradient update, respectively. 
    
    Note that in the first $n-r$ steps of gradient updates, the examples selected by \textbf{GD} are the same for both $S_1$ and $S_2$. Therefore, in this case, we have $G_t^m=G_t^M $. Thus, we can use Lemma 3.7 in \cite{ruder2016overview} to get 
    \begin{align}
        \delta_t \leq \delta_{t - 1}, t = 1, 2, \ldots, n - r + 1. \label{eq:r1}
    \end{align}
    In the $r$ subsequent gradient updates, the example selected by \textbf{GD} is different. In this case, we use the fact that both $G_t^m$ and $ G_t^M$ are $\eta_tL$-bounded as a consequence of Lemma 3.3 in \cite{ruder2016overview}. At this time, we have
    \begin{align}
    \delta_t  =&  \| G^m_{t - 1} (\theta^m_{t - 1}) - G^M_{t - 1} (\theta^M_{t
      - 1}) \|_2\notag\\
       \leq & \| \theta^m_{t - 1} - \theta^M_{t - 1} \|_2 + \| \theta^m_{t
      - 1} - G^m_{t - 1} (\theta^m_{t - 1}) \|_2
    + \| \theta^M_{t - 1} - G^M_{t -
      1} (\theta^M_{t - 1}) \|_2\notag \\
       \leq&\delta_{t - 1} + 2 \eta_t \cdot L.\label{eq:r2}
    \end{align}
    In summary, we can obtain $ \delta_{n + 1} \leq 2 L \sum_{t = n - r + 1}^{n } \eta_t $ by recursively applying Equation (\ref{eq:r1}) and Equation (\ref{eq:r2}).
    Consequently, for any $z=(x,y)$, we have the following inequality from the Assumption 2:
    \begin{align*}
    | f (\theta^m, z) - f (\theta^M, z) | \leq L \cdot \delta_{n + 1}
    \leq 2 L^2 \sum_{t = n - r + 1}^{n } \eta_t . 
    \end{align*}
    Furthermore, by left $\xi$-admissible condition, we obtain
    \begin{align}
    |m (x) - M (x)|  \leq 1 / \xi \cdot | f (\theta^m, z) - f
      (\theta^M, z) |
       \leq \frac{2 L^2}{\xi} \sum_{t = n - r + 1}^{n } \eta_t .\label{eq:mMbias}
    \end{align}
    Hence, we get $\| m - M \|_{L^2 (\mu)} =\left ( \mathbb{E}_X [(m(X)-M(X)^2] \right )^{\frac{1}{2} }  
    \leq  \frac{2 L^2}{\xi} \sum_{t = n
    - r + 1}^{n} \eta_t$. 
    Moreover, since $m$ can only change to $M$ under the current model change, so $\phi (\lambda)=0$.
    This implies that $ \| m - M \|_{L^2 (\mu)}$  is a $0$-subgaussian random variable. This proves that the model change under consideration can be viewed as a generally-occurring model change. Take
    $
    \delta =  \frac{2 L^2}{\xi} \sum_{t = n
    - r + 1}^{n} \eta_t, \nu = 0, \ell \rightarrow \infty
    $.
    By substituting them into Theorem \ref{thm2}, we finally get:
    \begin{align*}
     \Pr& \Bigg( \frac{1}{k} \sum_{i = 1}^k  m (X_i) - M (x)   \geq \frac{\gamma}{k}\sum_{i
      = 1}^k \| x - X_i \|_2 + \varepsilon \notag\\
    +& \left ( \frac{2 L^2}{\xi} \sum_{t = n
    - r + 1}^{n } \eta_t \right )  \cdot
      \kappa (\tilde{\mu}, \mu)\Bigg)
    \leq 2 \exp \left(- \frac{\varepsilon^2 k}{2}  \right).
    \end{align*}
    Thus, the proof is complete.$\hfill\square$
\end{proof}

As shown in Equation (\ref{eq:mMbias}), the model change from $m(\cdot)$ to $M(\cdot)$ is not a naturally-occurring model change. However, we have shown that it can be considered as a generally-occurring model change. Therefore, our simple example to a certain extent demonstrates that generally-occurring model changes have a wider range of applicability.
\section{Conclusion}
In this paper, we generalize the concept of naturally-occurring model change to a more general model parameter change concept, namely generally-occurring model change. We prove probabilistic guarantees for generally-occurring model change. In addition, we consider the specific problem of dataset perturbation, and use the probabilistic guarantees for generally-occurring model change to give relevant theoretical results based on optimization theory. This example shows that generally-occurring model change has a wider range of applications.

\textbf{Future Work.}
 However, there are still many areas for improvement in our work. First, it is an important problem to understand or improve the parameter $\kappa(\tilde{\mu},\mu)$ in Theorem \ref{thm2} in a more natural way. For this problem, we think it is a feasible idea to consider it from the perspective of information geometry. In addition, it is also a challenging problem to extend Theorem \ref{thm3} to more popular optimization algorithms such as \textbf{S}tochastic \textbf{G}radient \textbf{D}escent (\textbf{SGD}), instead of \textbf{GD}, as considered in Section \ref{sec:GS}.

\subsubsection{Acknowledgement.}
This work is supported by the National Key Research
and Development Program of China (Grant No.2022YFB3103702).

\bibliographystyle{splncs04}
\bibliography{ref}

\end{document}